%% file: main.tex
\def\year{2020}\relax
\documentclass[letterpaper]{article} 
\usepackage{aaai20}  
\usepackage{times}  
\usepackage{helvet} 
\usepackage{courier}  
\usepackage[hyphens]{url}  
\usepackage{graphicx} 
\usepackage{graphicx}  
\title{AAAI Press Formatting Instructions \\for Authors Using \LaTeX{} --- A Guide }

\usepackage[toc,page]{appendix}
\usepackage{multirow}
\usepackage{amsmath}
\usepackage{subfig}
\usepackage{amsthm}
\usepackage{algorithm}
\usepackage[noend]{algpseudocode}

\usepackage{comment}
\newtheorem{theorem}{Theorem}
\newtheorem{definition}[theorem]{Definition}
\newtheorem{lemma}[theorem]{Lemma}
\newtheorem{example}{Example}[section]
\usepackage{color}
\newcommand{\MLIC}{\ensuremath{\mathsf{MLIC}}}

\newcommand{\IMLI}{\ensuremath{\mathsf{IMLI}}}
\newcommand{\R}{\mathcal{R}}
\newcommand{\X}{\mathbf{X}}
\newcommand{\MaxSAT}{\ensuremath{\mathsf{MaxSAT}}}
\newcommand{\weight}[1]{\ensuremath{W \! \left( #1 \right)}}

\newcommand{\Rule}{\ensuremath{{\mathcal{R}}}}
\newtheorem{construction}[theorem]{Construction}

\title{IMLI: An Incremental Framework for  MaxSAT-Based Learning of  Interpretable  Classification Rules\thanks{The paper is published in the proceedings of  AAAI/ACM Conference on AI, Ethics, and Society (AIES 2019).}}

\author{Bishwamittra Ghosh \and Kuldeep S. Meel\\
National University of Singapore}

\begin{document}

\maketitle

\begin{abstract}
	\input{abstract.tex}
\end{abstract}
\input{introduction.tex}
\input{preliminaries.tex}

\input{problem_formulation.tex}
\input{incMLIC.tex}
\input{Experiment.tex}

\input{conclusion.tex}
\section*{Acknowledgments}
This work was supported in part by NUS ODPRT Grant R-$ 252 $-$ 000 $-$ 685 $-$ 133 $  and AI Singapore Grant R-$ 252 $-$ 000 $-A$ 16 $-$ 490 $. The computational work for this article was performed on resources of the National Supercomputing Centre, Singapore \url{https://www.nscc.sg}.

\clearpage

{
\bibliographystyle{aaai}
\bibliography{main} 
\input{Appendices.tex}

}
\end{document}

%% file: abstract.tex

The wide adoption of machine learning in the critical domains such as medical diagnosis, law, education had propelled the need for interpretable techniques due to the need for end users to understand the reasoning behind decisions due to learning systems. The computational intractability of interpretable learning led practitioners to design heuristic techniques, which fail to provide sound handles to tradeoff accuracy and interpretability. 

Motivated by the success of MaxSAT solvers over the past decade, recently MaxSAT-based approach, called {\MLIC}, was proposed that seeks to reduce the problem of learning interpretable rules expressed in Conjunctive Normal Form (CNF) to a MaxSAT query. While {\MLIC} was shown to achieve accuracy similar to that of other state of the art {\em black-box} classifiers while generating small interpretable CNF formulas, the runtime performance of {\MLIC} is significantly lagging and renders approach unusable in practice. In this context, authors raised the question: {\em Is it possible to achieve the best of both worlds, i.e., a sound framework for interpretable learning that can take advantage of MaxSAT solvers while scaling to real-world instances?} 

In this paper, we take a step towards answering the above question in affirmation. We propose {\IMLI}: an incremental approach to MaxSAT based framework that achieves scalable runtime performance via partition-based training methodology. Extensive experiments on benchmarks arising from UCI repository demonstrate that {\IMLI} achieves up to three orders of magnitude runtime improvement without loss of accuracy and interpretability.  

%% file: introduction.tex
\section{Introduction}

The recent advances in the machine learning techniques have led autonomous decision making systems be adopted in wide range of domains to perform data-driven decision making. As such the domains range from movie recommendations, ad predictions to legal, medical, and judicial. The diversity of domains mandate different criteria for the machine learning techniques. For domains such as movie recommendations and ad predictions, accuracy is usually the primary objective but for safety critical domains \cite{otte2013safe} such as medical and legal, interpretability, privacy, and fairness \cite{barocas2017fairness} are of paramount importance.

It has been long observed that the interpretable techniques are typically trusted and adopted by decision makers as interpretability provides them understanding of reasoning behind a tool's decision making \cite{ribeiro2016should}. At this point, it is important to acknowledge that formalizing interpretability is a major challenge \cite{doshi2017towards} and we do not claim to have final word on this. In this context, it is worth noting that for several domains such as medical domain, which was the motivation for our investigation, decision rules with small number of rules tend to be most interpretable \cite{letham2015interpretable}. 

Since the problem of rule learning is known to be in NP-hard, the earliest efforts focused on heuristic approaches that sought to combine heuristically chosen optimization functions with greedy algorithmic techniques. Recently, there has been surge of effort to achieve balance between accuracy and rule size via principled objective functions and usage of combinatorial optimization techniques such as  linear programming (LP) relaxations, sub-modular optimization, or Bayesian methods \cite{BertsimasCR2012,marchand2002SetCoverMachine,malioutov2013exact,boros2000-LogicalAnalysisOfData,wang2015or}
\footnote{An extensive survey of related work is presented in Appendix ~\ref{sec:related}.}.
 Motivated by the success of MaxSAT solving over the past decade, Malioutov and Meel proposed a MaxSAT-based approach, called {\MLIC} \cite{MM18}, that provides a precise control of accuracy vs. interpretability. The said approach was shown to provide interpretable Boolean formulas without significant loss of accuracy compared to the state of the art classifiers.  {\MLIC}, however, has poor scalability in terms of training time and times out for most instances beyond hundreds of samples. In this context, we ask: {\em Can we design a MaxSAT-based framework to efficiently construct interpretable rules without loss of accuracy and scaling to large real-world instances?}

The primary contribution of this paper is an affirmative answer to the above question. We first investigate the reason for poor scalability of {\MLIC} and attribute it to large size (i.e., number of clauses) of MaxSAT queries constructed by {\MLIC}. In particular, for training data of $n$ samples over $m$ boolean features, {\MLIC} constructs a formula of size $\mathcal{O}(n\cdot m \cdot k)$ to construct a $k-$clause Boolean formula. We empirically observe that the performance of MaxSAT solvers has worse than quadratic degradation in runtime with increase in the size of query. This leads us to propose a novel incremental framework, called {\IMLI}, for learning interpretable rules using MaxSAT. In contrast to {\MLIC}, {\IMLI} makes $p$ queries to MaxSAT solvers with each query of the size $\mathcal{O}(\frac{n}{p}\cdot m \cdot k)$. {\IMLI} relies on first partitioning the data into $p$ partitions and then incrementally learning rules on the $p$ partitions in a linear order such that rule learned for the $i$-th partition not only uses the current partition but regularizes itself with respect to the rules learned from the first $i-1$ partitions. We conduct a comprehensive experimental study over the large set of benchmarks and show that {\IMLI} significantly improves upon the runtime performance of {\MLIC} by achieving speedup of up to three orders of magnitude. Furthermore, the rules learned by {\IMLI} are significantly small and easy to interpret compared to that of the state of the art classifiers such as RIPPER and {\MLIC}.

Similar to Malioutov and Meel~\shortcite{MM18}, we hope that {\IMLI} will excite researchers in machine learning and CP/SAT (Constraint Programming/Satisfiability)  communities to consider this topic further: in designing new MaxSAT-based formulations and in turn designing the MaxSAT solvers tuned for interpretable machine learning.

%% file: preliminaries.tex
\section{Preliminaries}
\label{sec:preliminaries}

 We use capital  boldface letters such as $\mathbf{X}$ to denote matrices while lower boldface letters $\mathbf{y}$ are reserved for vectors/sets. For a matrix $\mathbf{X}$, $\mathbf{X}_i$  represents the $ i $-th row of $\mathbf{X}$ while for a vector/set $\mathbf{y}$, $y_i$ represents the $ i $-th element of $\mathbf{y}$. 

Let $F$ be a Boolean formula and  $\mathbf{b} = \{b_1,b_2,\dots ,b_m \}$ be the set of variables appearing in $F$. A literal is a variable ($b_i$) or its complement($\neg b_i$).  A \emph{satisfying assignment} or a
\emph{witness} of $F$ is an assignment of variables in $\mathbf{b}$ that makes
$F$ evaluate to \emph{true}.  If $\sigma$ is an assignment of variables and $b_i \in \mathbf{b}$, we use $\sigma(b_i)$ to denote the value assigned to
$b_i$ in $\sigma$. $F$ is in Conjunctive Normal Form (CNF) if $F := C_1 \wedge C_2 \dots \wedge C_k$, where each clause $C_i$ is represented as disjunction of literals. We use $|C_i|$ to denote the number of literals in $C_i$.   For two vectors $\mathbf{u}$ and $\mathbf{v}$ over propositional variables or constants ($ 0 $, $ 1 $, $ true $, $ false $ etc.), we define $\mathbf{u} \vee \mathbf{v} = \bigvee_{i} (u_{i} \wedge v_{i})$, where $u_{i}$ and $v_{i}$ denote a variable/constant  at the $i$-th index of $\mathbf{u}$ and $\mathbf{v}$ respectively. In this context, note that the operation $\wedge$ between a variable and a constant follows the standard interpretation, i.e., $0 \wedge b = 0$ and $1 \wedge b = b$.

We consider a standard binary classification, where we are given a collection of 
training samples $\{ \mathbf{X}_i, y_i \}$ where each vector $\mathbf{X}_i \in \mathcal{X}$ contains the valuation of the features $\mathbf{x} = \{x_1, x_2, \dots, x_m\}$ for sample $i$, and $y_i \in \{0,1\}$ is the binary label for sample $i$.
A classifier {\Rule} is a mapping that takes in a feature vector $\mathbf{x}$ and return a class $y$, i.e., $y = \Rule(\mathbf{x})$. The goal is not 
only to design $\Rule$ to approximate our training set, but also to generalize to unseen samples arising from the same distribution. We define two rules  $ \R_1 $ and $ \R_2 $ to be  equivalent  if $ \forall i,\R_1(\X_i)=\R_2(\X_i) $. 
In this work, we restrict $\mathbf{x}$ and $y$ to be Boolean (we discuss in Sect.~\ref{sec:nonbinaryfeatures} that such a restriction can be achieved without loss of generality) and focus on classifiers that can be expressed compactly in CNF. {We use $clause(\Rule,i)$ to denote the $i$-th clause of $\Rule$}. {Furthermore, we use $|\Rule|$ to denote the rule-size of classifier $ \Rule $ that is the sum of the count of literals in all the clauses, i.e., $|\Rule| = \Sigma_i |clause(\Rule,i)|$}.


In this work, we focus on the weighted variant of CNF wherein a weight function is defined over clauses. For a clause $C_i$ and 
weight function {\weight{\cdot}}, we use {\weight{C_i}} to denote the weight of clause $C_i$. We say that a clause $C_i$ is hard if $\weight{C_i} = \infty$, otherwise $C_i$ is called a soft clause.  To avoid notational clutter, we overload {\weight{\cdot}} to denote the weight of an assignment or clause, depending on the context. We define weight of an assignment $\sigma$ as the sum of weight of clauses that $\sigma$ does not satisfy. Formally, $\weight{\sigma} = \Sigma_{i | \sigma \not\models C_i} \weight{C_i}$.

Given $F$ and weight function $\weight{\cdot}$, the problem of {MaxSAT} is to find an assignment $\sigma^*$ that has the minimum weight, i.e.,  $\sigma^* = \MaxSAT(F,W)$ if $\forall \sigma \neq \sigma^*, \weight{\sigma*} \leq \weight{\sigma}$. Our formulation will have positive clause weights, hence MaxSAT corresponds to satisfying as many clauses as possible, and picking the strongest clauses among the unsatisfied ones. 
Borrowing terminology of community focused on developing {MaxSAT} solvers, we are solving a partial weighted {MaxSAT} instance wherein we mark all the clauses with $\infty$ weight as hard and clauses with other positive value less than $ \infty $ weight as soft  and ask for a solution that optimizes the partial weighted {MaxSAT} formula.  The knowledge of inner working of {MaxSAT} solvers and encoding of our representation into weighted {MaxSAT} is not required for this paper.

%% file: problem_formulation.tex
\section{Problem Formulation}
\label{sec:problem}

Given a training set $ \{\mathbf{X},\mathbf{y}\} $, our goal is to find an interpretable rule that is as accurate as possible. As noted earlier, there are several notions of interpretability. We follow the notion employed in Malioutov and Meel~\shortcite{MM18}, which focuses on the construction of rules involving few clauses each with few literals~\footnote{An advantage of Malioutov and Meel's formulation is a formal  notion of interpretability, which is amenable to formal analysis. We do not wish to claim that Malioutov and Meel's notion is the only formal definition of interpretability.}.

In particular, suppose 
    $ \mathcal{R} $ classifies all samples correctly, i.e., $ \forall i, y_i=\mathcal{R}(\mathbf{X}_i) $. Among all the rules that classify all samples correctly, we choose $ \mathcal{R} $ which is the sparsest (most interpretable) one. 
\[
\min\limits_{\mathcal{R}} |\mathcal{R}|\text{ such that }\forall i, y_i=\mathcal{R}(\mathbf{X}_i)
\]

 
 A classifier rule, however, can not classify all samples correctly. Hence we choose a classifier that makes less prediction error. $ \mathcal{E}_\mathcal{R} $   is the set of samples which are misclassified  by $ \mathcal{R} $, i.e., $\mathcal{E}_\mathcal{R}=\{\mathbf{X}_i | y_i \ne \mathcal{R}(\mathbf{X}_i) \}$. Hence we aim to find $ \mathcal{R} $ as follows. 
\[
\min\limits_{\mathcal{R}} |\mathcal{R}| + \lambda|\mathcal{E}_\mathcal{R}|  \text{ such that }\forall \mathbf{X}_i \notin \mathcal{E}_\mathcal{R} ,y_i = \mathcal{R}(\mathbf{X}_i) 
\]

$ \lambda $ is the data fidelity  parameter balancing the trade-off between classifier complexity  and prediction accuracy.  
Higher value of $ \lambda $ guarantees less prediction error while sacrificing the  sparsity of $ \mathcal{R} $ by adding more literals in $ \mathcal{R} $, and vice versa. Therefore $ \lambda $ is  an inverse of regularization. 

%% file: incMLIC.tex
\section{ {\IMLI}: MaxSAT-Based Incremental Learning Framework of Interpretable Rules}\label{sec:imli} 

In this section, we present the primary contribution of this paper, {{\IMLI}}, which is a MaxSAT-based incremental learning framework for interpretable classification rules. The core technical idea behind {{\IMLI}} is to divide the training data into a fixed number $ p $ of partitions {and employ {\MaxSAT} based learning framework for each partition } such that the MaxSAT query constructed for partition $i$ is based on the training data for partition $i$ and the rule learned until partition $i-1$.
 To this end, we use the notation $(\mathbf{X}^i, \mathbf{y}^i) $ to refer to the training data for the $i$-th partition. We assume that  $\forall i, |\mathbf{X}^i| = |\mathbf{X}^{i-1}|$.

The rest of the section is organized as follows: we first describe the construction of {MaxSAT} query for the  $i$-th partition  in Sect.~\ref{subsec:query} to learn CNF rules, and then discuss the discretization techniques for real-world datasets in Sect.~\ref{sec:nonbinaryfeatures}. The incrementality of {{\IMLI}} gives rise to the challenge of having redundant literals in the learned rules; we address such redundancy in Sect.~\ref{sec:pruning} and finally we discuss, in Sect.~\ref{sec:dnf}, how our framework for learning CNF rules can be easily extended to learn DNF rules as well.  
We provide an illustration of rule learning of {\IMLI} in Appendix ~\ref{sec:example}.


	\subsection{Construction of MaxSAT Query}\label{subsec:query}

  We now discuss the construction of a  {MaxSAT} query, denoted by $Q_i$, for the  $i$-th partition ($i \in [1,p]$). To construct the {MaxSAT} query for the $i$-th partition, we assume an access to the rule learned from the  $(i-1)$-th partition (where $\R_0$ is  an empty formula). 
  
 The construction of $Q_i$ takes in four parameters: (i) $ k $, the desired number of clauses in CNF rule, (ii) $ \lambda $, the data fidelity parameter, (iii) a matrix ${\mathbf{X}^{i}} \in \{0,1\}^{n \times m}$  describing the  binary value of  $m$ features for each of    $n$ samples with $\mathbf{X}^{i}_q$ being a binary valued vector for the $ q $-th sample corresponding to feature vector $\mathbf{x} = \{x_1, x_2, \dots, x_m\}$, (iv) a label vector $\mathbf{y}^{i} \in \{0,1\}^n$ containing a class label $y^{i}_q$ for the  sample $\X^{i}_q$.   Consequently, {{\IMLI}} constructs a {MaxSAT} query for the $ i $-th partition and invokes an off-the-shelf MaxSAT solver to compute the underlying rule $ \R_i $.

	{\IMLI} considers two types of propositional variables: (i) feature variables and (ii) noise (classification error) variables.
 	For the $ i $-th partition, {\IMLI} formulates a classifier rule $ \R_i $ based on following intuition.  Recall, a $ k $-clause CNF rule $ \R_i=\bigwedge_{l=1}^kC_l $ is  represented as the conjunction of  $ k $ clauses where clause $ C_l $ is the disjunction of  feature variables. A  sample $ \X^i_q $ satisfies  $ C_l $ if $ \X^i_q $ has at least one similar feature whose representative variable is present in $ C_l $. If $ \X^i_q $ satisfies  $ \forall l, C_l$, then  $ \R_i(\X^i_q)=1 $ otherwise $ \R_i(\X^i_q)=0 $. Since  feature $ x_j $ can be present or not present  in each of   $ k $ clauses, {\IMLI} considers $ k $  boolean variables, each denoted by $ b_j^l $ ($ l \in [1,k] $) for  feature $ x_j $ to denote its participation in the $ l $-th clause. A sample $ \X^i_q $, however, can be misclassified by $ \R_i $ i.e., $ \R_i(\X^i_q)\oplus  y_q^i =1 $.  {\IMLI}  introduces  a noise variable $ \eta_q $ corresponding to  sample $ \X^i_q $ so that the assignment of $ \eta_q $ can be interpreted whether $ \X^i_q $ is misclassified by $ \R_i $ or not.  Hence the key idea of {{\IMLI}} for learning the $ i $-th partition is to define a {MaxSAT} query over $k \times m + n$ propositional variables, denoted by $\{ b_{1}^{1}, b_{2}^{1}, \dots ,b_{m}^{1}, \dots ,b_{m}^{k}, \eta_1, \dots ,\eta_n\}$.   The {\MaxSAT} query of {\IMLI} consists of the following three sets of constraints:

	
%
%

%
	\begin{enumerate}
		
		\item Since our objective is to find sparser rules, the default objective of  {\IMLI} would be to add a constraint to falsify as many $ b_j^l $  as possible. 	As noted earlier,  rule $\Rule_{i-1}$ from the $(i-1)$-th partition plays an important role in the construction of MaxSAT 
		constraints of the $i$-th partition. Therefore,   if $x_{j} \in clause(R_{i-1},l)$, {\IMLI} would deviate from its default behavior by adding a constraint to keep the corresponding literal $ true $ in the optimal assignment.  The weight corresponding to this clause is $ 1 $. We formalize our discussion as follows:
	\begin{equation*}
	V^l_j:= 
	\begin{cases}
	b^l_j&\quad \text{if }  x_{j} \in clause(\R_{i-1},l)\\
	\neg b^l_j&\quad \text{otherwise } \\
	\end{cases}
	;\quad  W(V^l_j)= 1
	\label{eq:constraint-v}
	\end{equation*}
		\item We use noise variables to handle mis-classifications and therefore, {\IMLI} tries to falsify as many   noise variables as possible. Since data fidelity parameter $ \lambda $ is proportionate to accuracy, {\IMLI} puts $ \lambda $ weight to each following soft  clause. 
		\begin{equation*}
		N_q:= (\neg \eta_q ); \qquad \qquad W(N_q) = \lambda
		\label{eq:constraint_two}
		\end{equation*}
		
		\item 
		Let $\mathbf{B}_{l} = \{b^{l}_{j} \mid j \in [1,m] \}$. 
		Here we provide the third set of constraints of {\IMLI}. 
		\begin{equation*}
		D_q:= (\neg \eta_q \rightarrow ( y_q \leftrightarrow \bigwedge_{l=1}^{k} ({\mathbf{X}_q} \vee {\mathbf{B}_{l}})));    W(D_q) = \infty
		\label{eq:constraint_three}	
		\end{equation*}
		Every hard clause $D_q$  can be interpreted as follows. If $\eta_q$ is assigned to \textit{false} ($ \neg \eta_q = true$) then ${y}^i_q = \R_i(\X^i_q)=\bigwedge_{l=1}^{k} ({\mathbf{X}_q} \vee {\mathbf{B}_{l}})$.  The operator ``$\vee  $'' is defined in Sect. \ref{sec:preliminaries}.  
		

	\end{enumerate}
	Finally, the set of constraints $Q_i$ for  the $ i $-th partition constructed by {{\IMLI}} is defined as follows:
	\begin{equation*}
	\label{eq:constraint-imli}
	Q_i :=  \bigwedge_{j=1, l = 1}^{j=m,l=k} V^l_j \wedge \bigwedge_{q=1}^{n} N_q \wedge \bigwedge_{q=1}^{n} D_q
	\end{equation*} 
	
	Next, we extract $\Rule_i$ from the solution of $Q_i$ as follows. 
	\begin{construction}
		Let $\sigma^* = \MaxSAT(Q_i,W)$, then $x_j \in {clause(\Rule_i,l)}$ iff $\sigma^*(b_{j}^{l}) = 1$.
	\end{construction}
	In the rest of the manuscript, we will use $\Rule$ to denote $\Rule_{p}$.  
	\subsection{Beyond Binary Features}\label{sec:nonbinaryfeatures}

	We have considered that the feature value of a training sample is binary. Real-world  datasets, however,  contain categorical, real-valued or numerical features. We use the standard discretization technique to convert categorical and continuous (real or integer value) features to boolean features. We use  {\em one hot encoding} to convert categorical features to binary features by introducing a boolean vector  with the cardinality equal to the number of distinct categories of individual categorical features.
	\begin{example}
		\label{ex:bin_categorical}
		Consider a categorical feature with three categories: ``red'', ``green'', ``'yellow''. One hot encoding would convert this feature to three binary variables, which take values $ 100 $, $ 010 $, and $ 001 $ for the three categories.
	\end{example}
	
	 Furthermore, we can discretize the continuous-valued features into  binary features by comparing the feature value to  a collection of  thresholds within range and introducing  a boolean feature vector with cardinality proportional to the number of considered thresholds \cite{MM18}.
Specifically, for a continuous feature $ x_c $ we consider a number of thresholds $ \{\tau_1,\dots,\tau_t\}$ where $  \tau_i < \tau_{i+1}  $ and define two separate Boolean features $ I[x_c \ge \tau_i] $ and $ I[x_c < \tau_i] $ for each $ \tau_i $. 
	We present the following definitions based on the discretization of  continuous features.
	
		 	\begin{definition}
		 		$ tval(b): b \rightarrow \tau $ is a function over boolean variables corresponding to discretized binary features (from a continuous feature) and outputs the compared threshold value. 
		 	\end{definition}
		 	
		 	\begin{definition}
		 		$ op(b): b \rightarrow \{\ge , <\} $ is a function over boolean variables corresponding to discretized binary features (from a continuous feature) and outputs the comparison operator between continuous feature value and $ tval(b) $. 
		 	\end{definition}
		 	
		 	\begin{definition}
		 		$ siblings(b_i, b_j): (b_i, b_j) \rightarrow \{true, false\} $ is a function over pair of  boolean variables $ b_i,b_j $ and outputs $ true $ if the boolean features corresponding to $ b_i,b_j $ are constructed by discretizing the same continuous feature  
		 		and  
		 		$ op(b_i)=op(b_j) $.
		 	\end{definition}
	 	
	 	\begin{example}
	 		\label{ex:bin-continuous}
	 		Consider  a  continuous feature $ x_c $ with range $  (0,100) $ and three thresholds $ \{25,50,75\} $ associated with this feature.  $ \IMLI $ introduces $ 6 $ new boolean features   $ \{ x_1: I[x_c\ge 25],x_2: I[x_c\ge 50], x_3:I[x_c\ge 75], x_4:I[x_c< 25], x_5:I[x_c< 50],x_6:I[x_c < 75]\} $. Following this discretization technique,  the binary feature vector  of a sample  with   feature value $ x_c=37.5 $ is  $ 100011 $, because among the $6  $ introduced boolean features $ x_1:I[ 37.5 \ge 25 ]=1$, $ x_5:I[37.5 < 50]=1 $, and $x_6:I[ 37.5 < 75]=1 $.

	 	\end{example}
	 	
	 	\begin{example}
	 		\label{ex:prune}
	 		In Example \ref{ex:bin-continuous}, $ b_i $ is a boolean variable corresponding to feature $ x_i $. Now $ tval(b_1)=25$, $ op(b_1)=\text{``}\ge\text{''}$, $ siblings(b_1,b_2)=true, $ and $ siblings(b_1,b_4)=false $. 
	 	\end{example}

		

	\subsection{Redundancy Removal}\label{sec:pruning}
	Given the incremental procedure of learning {\Rule} where the constraints for the $i$-th partition are influenced from the rule learned until the $(i-1)$-th partition, one key challenge is to address potential redundancy in the learned rules. In particular, we observe that redundancy manifests itself in binary features corresponding to continuous-valued features as the $(i-1)$-th partition  might suggest inclusion of feature  $ I[x_c < \tau_u] $ while the $ i $-th partition  also suggests inclusion of feature $ I[x_c < \tau_v] $ where $\tau_u \ne \tau_v$. To this end, we present Algorithm \ref{alg:imli-RRL} to remove redundant literals. 

	\begin{algorithm}
		\caption{Remove Redundancy}\label{alg:imli-RRL}
		\begin{algorithmic}[1]
			\Procedure{removeRedundantLiterals}{$\R$}
			\For{each  clause $ C_l $ of $ \R $} 
			\For{ each pair $ \langle b_i^l,b_j^l \rangle $  where   $ \sigma(b_i^l)=\sigma(b_j^l)=1$,   $siblings(b_i^l,b_j^l)=true $, and $ tval(b_i^l) < tval(b_j^l) $ }
			\If {$op(b_i^l)=op(b_j^l)=\text{``}\ge \text{''} $}
			\State  $ \R'=\R[\sigma(b_j^l)\mapsto0] $ \Comment{$ b_j^l $ is redundant}
			\Else \State $ \R'=\R[\sigma(b_i^l)\mapsto0] $
			\EndIf
			\EndFor
			\EndFor
			\State \Return $ \R' $
			\EndProcedure
		\end{algorithmic}
	\end{algorithm}

	\begin{lemma}
		 $ |\R'| \le |\R| $ and $ \R' $ is equivalent to $ \R $. 
		
		\label{lm:adjustment_of_rule_size}
	\end{lemma}
\begin{proof}
	See Appendix~\ref{prf:prune}. 
\end{proof}

	\subsection{Learning DNF Rules}
	\label{sec:dnf}

			Primarily we focus on learning rule $ \R $ which is in CNF form. We can also apply incremental technique for learning DNF rules.   Suppose, we want to learn a rule $ y=S(x) $ where $ S(x) $ is expressible in DNF. We show that $ y = S(x) \leftrightarrow \neg (y=\neg S(x)) $. Here $ \neg S(x) $ is in CNF.  Therefore, to learn DNF rule $ S(x) $, we simply call {\IMLI} with $ \neg \mathbf{y} $ as input  for all $ p $ batches, learn CNF rule, and finally
	negate the learned rule.  Hence Algorithm \ref{alg:imli-RRL} can be directly applied.
	\begin{example}
		\label{ex:dnf_rule}
		``$( \text{is Male}\vee \text{Age}<\text{50})\wedge(\text{Education}=\text{Graduate}\vee\text{Income}\ge\text{1500}) $''\textemdash rule is  learned for negated class label. The resultant DNF rule is  ``$ (\text{is not Male}\wedge  \text{Age}\ge\text{50}) \vee (\text{Education} \ne \text{Graduate}\wedge\text{Income}<\text{1500} )$'' 
	\end{example}

%% file: Experiment.tex
\section{Experiment}
\label{sec:experiment}

	 We have implemented a prototype implementation in Python to evaluate the performance of \IMLI
	 \footnote{ \url{https://github.com/meelgroup/mlic}}. The experiment has been  conducted on high performance computer
	 cluster, where each node consists of E$ 5$-$2690 $ v$ 3 $ CPU with $ 24 $ cores, $ 96 $GB of RAM, and in total  130,000 CPU hours. We have conducted an extensive set of experiments on publicly available benchmarks 
	 (detailed description in Appendix ~\ref{sec:dataset}) 
	 from UCI repository \cite{Dua:2017} to answer the following questions.
	 
	 \begin{enumerate}
		 \item How do the training time and accuracy of {\IMLI} compare to that of  state of the art classifiers including both interpretable and non-interpretable ones?
	     \item How do  accuracy, rule size, and  training time of {\IMLI} vary with data fidelity parameter $ \lambda $ and the number of partitions $ p $?
	     \item How interpretable are the rules generated by {\IMLI}? 

	 \end{enumerate}

	 In summary, the experimental results demonstrate that {\IMLI} can scale to large datasets involving tens of thousands of samples with hundreds of binary features. In contrast to {\MLIC}, {\IMLI} achieves up to three orders of magnitude improvement in training time without loss of accuracy and interpretability. 
	 {\IMLI} generates rules which are not only interpretable but also accurate compared to other classifiers, which often produce non-interpretable models for the sake of accuracy. 


%
	\subsection{Experiment Methodology:}
	To measure the performance gain over {\MLIC}, we measure the accuracy and training time of {\IMLI} vis-a-vis {\MLIC}. We also perform comparisons with another state of the art classifier RIPPER and other (mostly) non-interpretable classifiers such as random forest (RF), support vector classifier (SVC),  Nearest Neighbors classifier (NN), $ l_1 $-penalized Logistic Regression (LR). 
	
	The number of parameter values is comparable $ (10) $ for each technique. For RF and RIPPER, we use control based on the cutoff of the number of examples in the leaf node. For SVC, NN,  and LR we discretize the regularization parameter on a logarithmic grid. For both  $ \IMLI $ and $ \MLIC $, we have two choices of $ \lambda \in \{5,10\} $, three choices of $ k\in \{1,2,3\} $, and two choices of the type of rule as $ \{\text{CNF}, \text{DNF}\} $. For $ \IMLI $ we vary the number of partitions $ p $ for each dataset  such that each partition has at least eight samples and at most $ 512 $ samples. For all classifiers, we set the training time  cutoff to be $ 1000 $ seconds.

	We perform an assessment of test  accuracy on a holdout set and mean validation accuracy on a $ 10 $-fold cross-validation set (holdout set $ 10 \% $, validation set $ 9\% $, training set $ 81\% $).
	 We compute test accuracy and  mean  validation accuracy across the ten folds for each choice of the parameters for each technique, and report  test accuracy,   mean validation accuracy, and mean training time for a choice of the parameters which incurs the  best test accuracy. To remove the bias of a particular holdout set we perform ten repetitions with different holdout sets and present the mean statistics.   
	 

	
	For {\MLIC} and {\IMLI}, we experimented with different {MaxSAT} solvers and finally chose MaxHS \cite{davies2011solving} for {\MLIC} since {MaxSAT} queries generated by {\MLIC} timeout for all the solvers and MaxHS is the only solver to return the best found answer so far. In contrast, queries constructed by {\IMLI} are easier and the best runtime performance is obtained by using Open-WBO solver~\cite{martins2014open}.   

	{
		\begin{table*}[ht]
			\begin{center}
				\begin{tabular}{|l | r | r |r |r |r| r |r |r |r| }
					\hline
					
					\textbf{Dataset} & \textbf{Size} & \textbf{Features} & LR & NN & RF &    SVC  &RIPPER& $ \MLIC $ & {$ \IMLI $} \\
					
					\hline
					\multirow{3}{*}{Parkinsons} & \multirow{3}{*}{195} & \multirow{3}{*}{392} & $ 97.5$ & $ 90$ & $ 97.5$ &$ 87.5 $ & $ 85.00$ & $ 97.5 $ &  $ 95 $
					\\
					&&&  $[ 69.76 ]$  &  $[  82.68 ]$  &  $[  80.98 ]$ &  $[  83.5 ]$ &  $[  76.74 ]$   &  $[ 82.35]   $  & $ [79.41] $
					
					\\ 
					&&&   $( 0.22 )$ & $(  0.14 )$ &  $(  1.7)$ &  $(  0.18  )$&  $(  2.92 )$  &  $( 114.75  )$ &   $ (0.37) $\\
					\hline
					\multirow{3}{*}{Ionosphere } & \multirow{3}{*}{351} & \multirow{3}{*}{540}  &$ 93.06 $&$ 86.11 $&$ 95.83 $&$ 95.83 $&$ 93.06 $&$ 93.06$ & $ 91.67 $\\
					&& & $ [90.64] $&$ [87.18] $&$[92.77]  $&$ [90.65] $&$ [85.73] $& $ [91.94] $  & $ [85.48] $ \\
					&& & $ (0.32) $&$ (0.26) $&$ (1.72) $&$ (0.26) $& $ (3.3) $& $ (917.13) $ & $ (0.5
					) $\\
					\hline		
					
					\multirow{3}{*}{WDBC} & \multirow{3}{*}{569} & \multirow{3}{*}{540} & $ 98.28 $&	$ 95.69 $&	$ 96.55 $&$ 96.55 $&	$ 92.24 $		& $ 93.97 $	   & $ 89.66 $
					\\
					&&&  $ [96.77] $	&$ [97.27] $	&$ [96.68] $&$ [97.16] $
					&	$ [93.54] $	& $ [95.1] $ & $ [91.18] $
					\\ 
					& & & $ (0.33) $	&$ (0.45) $ & 	$ (1.8) $  & 	$ (0.27) $&	$ (3.53) $ & Timeout &  $ (0.78
					) $
					\\
					\hline
					\multirow{3}{*}{Blood} & \multirow{3}{*}{748} & \multirow{3}{*}{64} & $ 80 $& $76.67 $& $76 $& $76 $& $76
					$&  $ 75.33 $
					&$ 76 $\\
					&&& $ [75.92] $&$ [76.14] $& $[76.22] $& $[76.22] $& $[76.22]$ & $ [77.61] $ & $ [76.12] $
					\\
					&&& $(0.2 )$&$( 0.2 )$&$( 1.68)$& $ (0.18) $ &$( 2.23
					)$  & $ (5.96) $ &  $ (0.24) $
					\\
					\hline
					\multirow{3}{*}{PIMA} & \multirow{3}{*}{768} &
					\multirow{3}{*}{ 134} &$75.32 $ &$77.92$ &$76.62$ &$ 75.32$&$ 75.32
					$ &$ 75.97 $& $ 73.38 $\\
					&&& $[74.75]$ &$[73.23 ]$&$[75.54]$ &$[76.63 ]$&$ [74.36
					] $ &$ [71.74] $&  $ [68.12] $\\
					&& &$(0.3 )$&$( 0.32 )$&$( 1.99 )$ &$( 0.37 )$&$ (2.58
					) $ &Timeout&  $ (0.74
					) $\\
					\hline 
					\multirow{3}{*}{Tom's HW } & \multirow{3}{*}{28179} &
					\multirow{3}{*}{  844 } &$96.98 $ &$ 94.11  $&$97.11 $  &$ 96.83 $ &$ 96.75
					$& $ 96.61 $& $96.86 $\\
					&& &$[97.12]$ &$[ 93.91 ]$&$[97.35]$& $[ 97.1]$&$ [97.12] $ &$ [96.55] $& $ [96.23] $\\
					&& &$(2.24)$ &$(910.36)$ &$(27.11)$ &$(354.15)$&$ (37.81)
					$&Timeout&  $ (23.67) $\\
					\hline
					\multirow{3}{*}{Adult} & \multirow{3}{*}{32561} &
					\multirow{3}{*}{262} & $84.58
					$ &$ 83.46 $&$84.31$ & $ 84.39$&$ 83.72
					$ &$ 79.72 $& $ 80.84 $\\
					&&& $[84.99
					]$ &$[ 83.62]$ &$[84.68]$ &$[ 84.69]$&$ [83.49
					] $ &$ [79.53] $& $ [77.43] $\\
					&&& $(5.8
					)$&$( 640.81)$&$(36.64)$ &$(918.26)$&$ (37.66) $ &Timeout& $ (25.07
					) $\\
					\hline
					
					\multirow{3}{*}{Credit-default} & \multirow{3}{*}{30000} &
					\multirow{3}{*}{ 334  } &$80.81
					$&$ 79.61 $&$80.87$ & $80.69$&$ 80.97
					$& $ 80.72 $ & $ 79.41 $\\
					&& &$[82.16
					]$&$[ 80.2 ]$&$[82.13 ]$&$[82.1]$& $ [82.07] $ & $ [82.14] $ &  $ [75.81] $\\
					&& &$(6.87
					)$ &$(872.97 )$ &$(37.72 )$ & $(847.93  )$&$ (20.37
					) $& Timeout  & $ (32.58) $\\
					\hline
					\multirow{3}{*}{Twitter} & \multirow{3}{*}{49999 } &
					\multirow{3}{*}{ 1050 } &$95.67
					$ & \multirow{3}{*}{Timeout}  &$95.16 $&\multirow{3}{*}{Timeout}&$ 95.56
					$  & $ 94.78 $& $ 94.69 $ \\
					&& &$[96.32
					]$ & &$[96.46]$ &&$ [96.16
					] $  & $ [95.69] $ &  $ [95.08] $\\
					&& &$(3.99
					)$&&$(67.83  )$&&$ (98.21
					) $ & Timeout&$ (59.67	) $\\
					\hline

				\end{tabular}
			\end{center}
			\caption{Comparisons of classification accuracy with $ 10 $-fold cross validation for different classifiers. For every cell in the last {seven} columns the top value represents the test accuracy $ (\%) $ on unseen data, the middle value surrounded by square bracket represents average validation accuracy $ (\%) $ of $ 10 $-fold, and the bottom value surrounded by parenthesis represents the average training time in seconds. }
			\label{tab:results}
		\end{table*}
	}

	 	\begin{table}[ht]
	 		\vspace{-20pt}
	 		\begin{center}
	 			\setlength\tabcolsep{4.3pt}
	 			\begin{tabular}{|l|   r | r| r|}
	 				\hline
	 				Dataset& {RIPPER} & $ \MLIC $& $ \IMLI $ \\

					\hline
					Parkinsons   & $ 2.6 $ & $2$& $ 8 $\\
					\hline 
					Ionosphere   & $ 9.6 $& $ 13 $& $ 5 $ \\
					\hline 
					WDBC  & $ 7.6 $&  $ 14.5
					$  & $ 2 $ \\
					\hline
					Blood  &$ 1 $ & $ 3 $ & $ 3.5 $ \\
					\hline
					Adult    & $ 107.55$ & $ 44.5 $  & $ 28 $\\
					\hline 
					PIMA  & $ 8.25 $ & $ 16
					$ & $ 3.5 $ \\
					\hline 
					Tom's HW  & $ 30.33 $ & $ 2
					$	& $ 2.5 $ \\
					\hline
					Twitter   & $ 21.6 $ & $ 20.5
					$ & $ 6 $\\
					\hline
					Credit & $ 14.25 $ & $ 6
					$ & $ 3 $\\
					\hline
									
	 			\end{tabular}
	 		\end{center}
	 		\caption{Size of the rule of interpretable classifiers.}
	 		\label{tab:rule_size}
	 	\end{table}
	 \subsection{Results}
	 \subsubsection{Comparison Among Different Classifiers:}
		Table \ref{tab:results} presents the comparison of $ \IMLI $ vis-a-vis typical interpretable and non-interpretable classifiers. The first three columns list the name, size (number of samples), and the number of  binary features (discretized) for each dataset. The next seven columns  present test accuracy, validation accuracy, and training time of  the classifiers.

		In  Table \ref{tab:results} we observe that $ \MLIC $ and RIPPER have slightly higher accuracy than $ \IMLI $.  Specifically  considering all  datasets  $ \MLIC $ (resp. RIPPER) has on average $ 1.12 \% $ (resp. $ 0.12 \% $) higher test accuracy and $ 3.09 \% $ (resp. $ 2.29 \% $) higher validation accuracy than that of $ \IMLI $. In contrast, {\IMLI} takes up to three order of magnitude less training time compared to  $ \MLIC $ and upto one order of magnitude less time  compared to  RIPPER. Interestingly, {\IMLI} is competitive to black-box classifiers, e.g.  SVC and NN for large datasets. In this context, we think {\IMLI} achieves a sweet spot in achieving significant runtime improvement in training without losing accuracy. 
		
		At this point, one may wonder as to whether minor loss in accuracy also leads to loss of interpretability. To this end, we illustrate a detailed comparison among the generated rules of  $ \IMLI $, RIPPER, and $ \MLIC $ in Table \ref{tab:rule_size}. We observe that rule size of {\IMLI} is significantly smaller than that of RIPPER and {\MLIC}. In particular, note that {\IMLI} can generate rules with size less than eight for all the datasets (exception in Adult dataset where {\IMLI} still has the most sparse rule), thereby demonstrating the sparsity of generated rules. In contrast, {\MLIC} and RIPPER generate rules of significantly larger size than {\IMLI}. As indicated earlier, sparsity is only one of several possible approaches to quantify interpretability. Therefore, we also decided to observe the generated rules and interestingly, the generated rules seem very intuitive. 
		We have listed the generated rules in Appendix~\ref{sec:rules}

\begin{figure}
	\centering
	\subfloat
	[$ k \in \{1,2\}, p=4 $]{\includegraphics[width=0.23\textwidth]{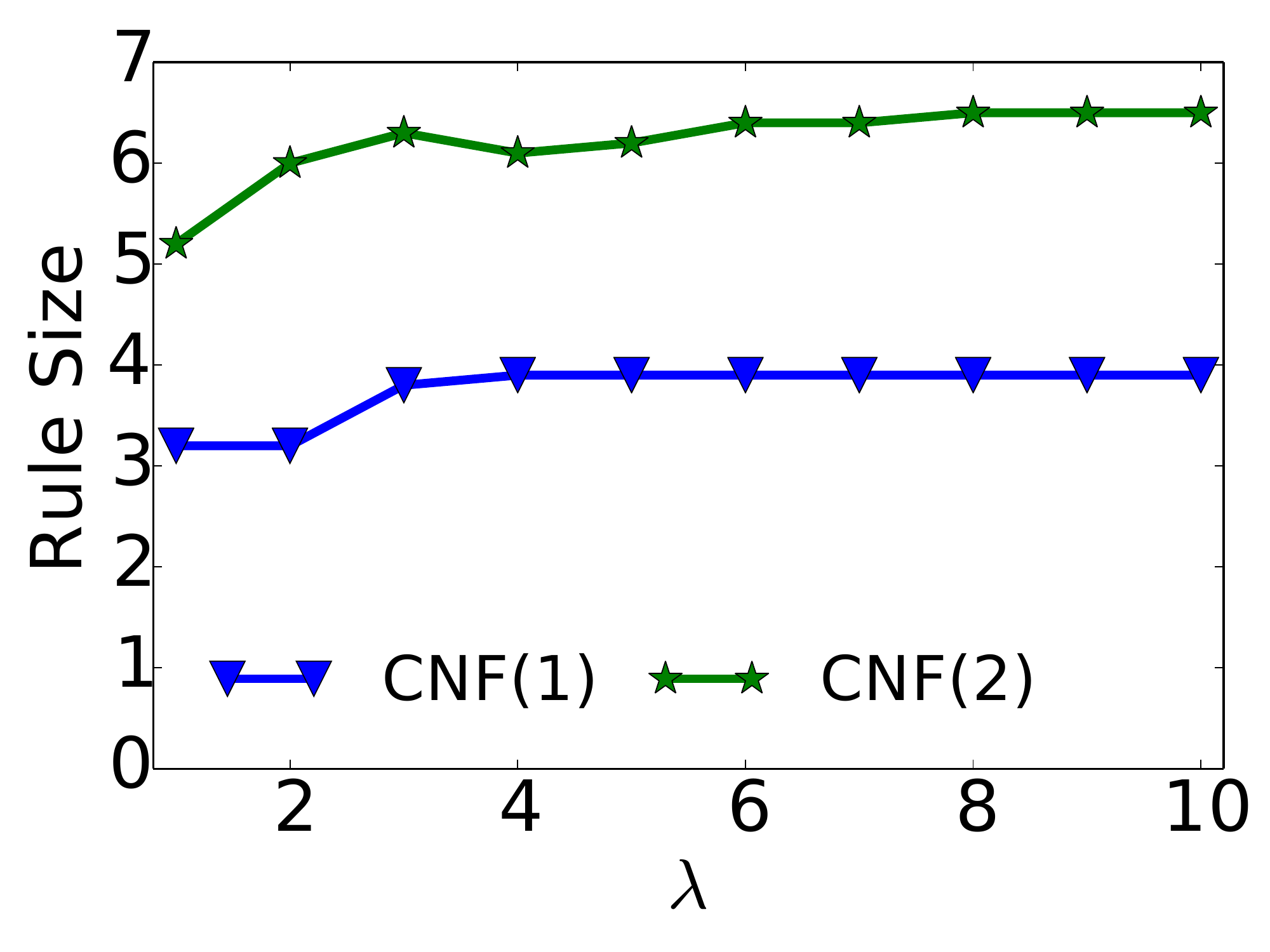}\label{fig:lambda(a)}} \hfill
	\subfloat[$ k \in \{1,2\}, p=4 $]{\includegraphics[width=0.23\textwidth]{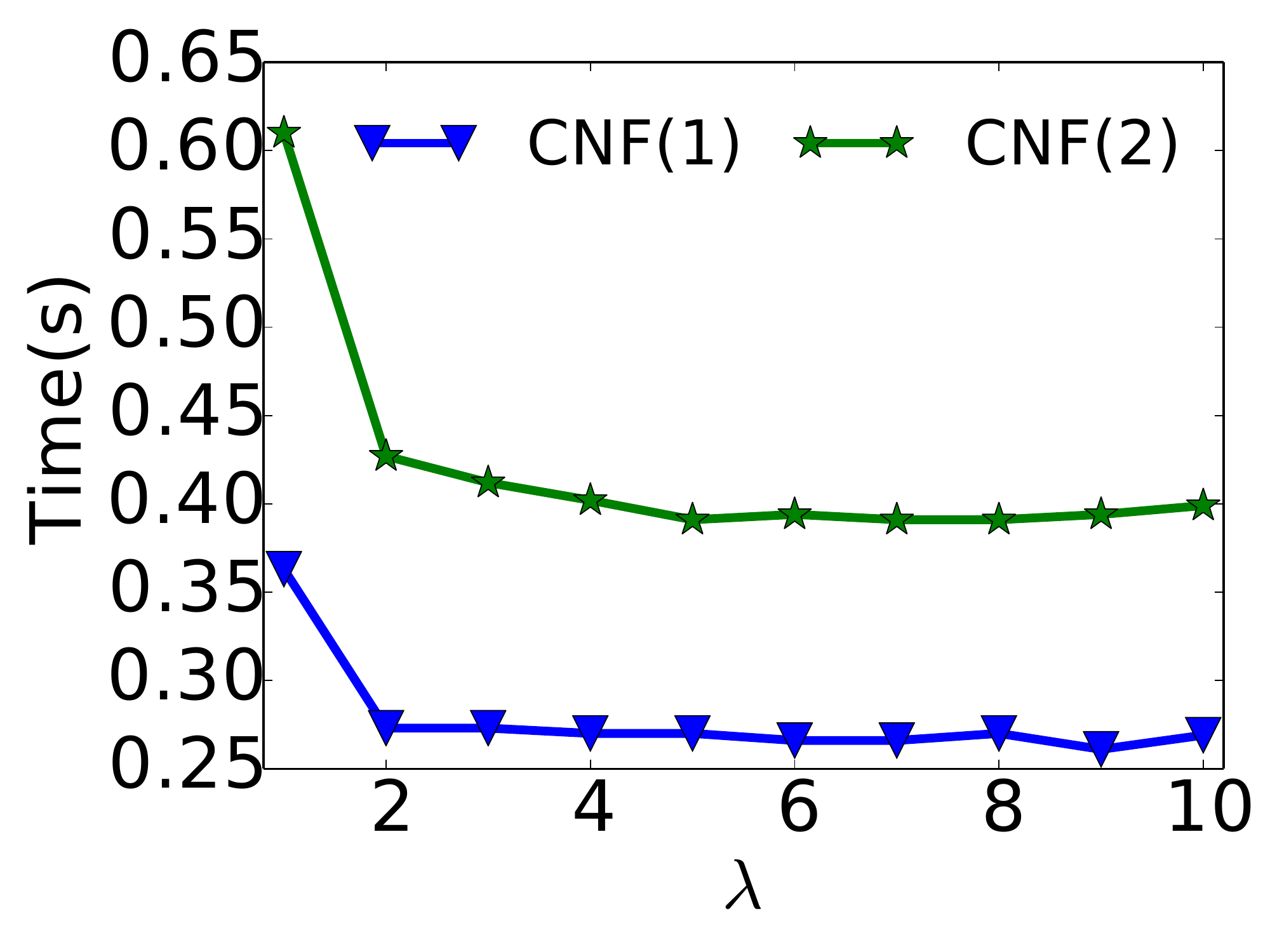}\label{fig:lambda(b)}} \hfill
	\caption{Effect of data fidelity parameter $\lambda$ on rule size and  training time. The number within parenthesis denotes the number of clause  $ k $ for the respective rule. }
	\label{fig:lambda}
\end{figure}
\subsubsection{Varying Data Fidelity $ \lambda $:}
In Figure \ref{fig:lambda} we present the result for varying  $ \lambda $. 
Our experiment result finds a similar observation in all the datasets, and here we present result for Parkinsons dataset.


Recall that size of a rule is the total number of literals  appearing in $ \Rule $.
As we increase the value of $ \lambda $, rule size (Figure \ref{fig:lambda(a)}) and the time taken to solve the MaxSAT query (Figure \ref{fig:lambda(b)}) decreases. When $ \lambda=1 $, all the soft clauses have equal weight.
  However, when   $ \lambda $ is higher, soft clause $ N_q $ is put a higher weight than $ V_j^l $, which turns out in finding the solution of the query requiring less time because of the priority among soft clauses. Therefore, the generated rule becomes sparser.  We find a similar trend for  DNF rules too. 
  In empirical study we find that
  as we increase $ \lambda $, training accuracy increases gradually but validation  accuracy and test accuracy do not follow a monotonic behavior in the partition-based learning.

\begin{figure}
	\centering
	\vspace{-8pt}
	\subfloat[]{\includegraphics[width=0.23\textwidth]{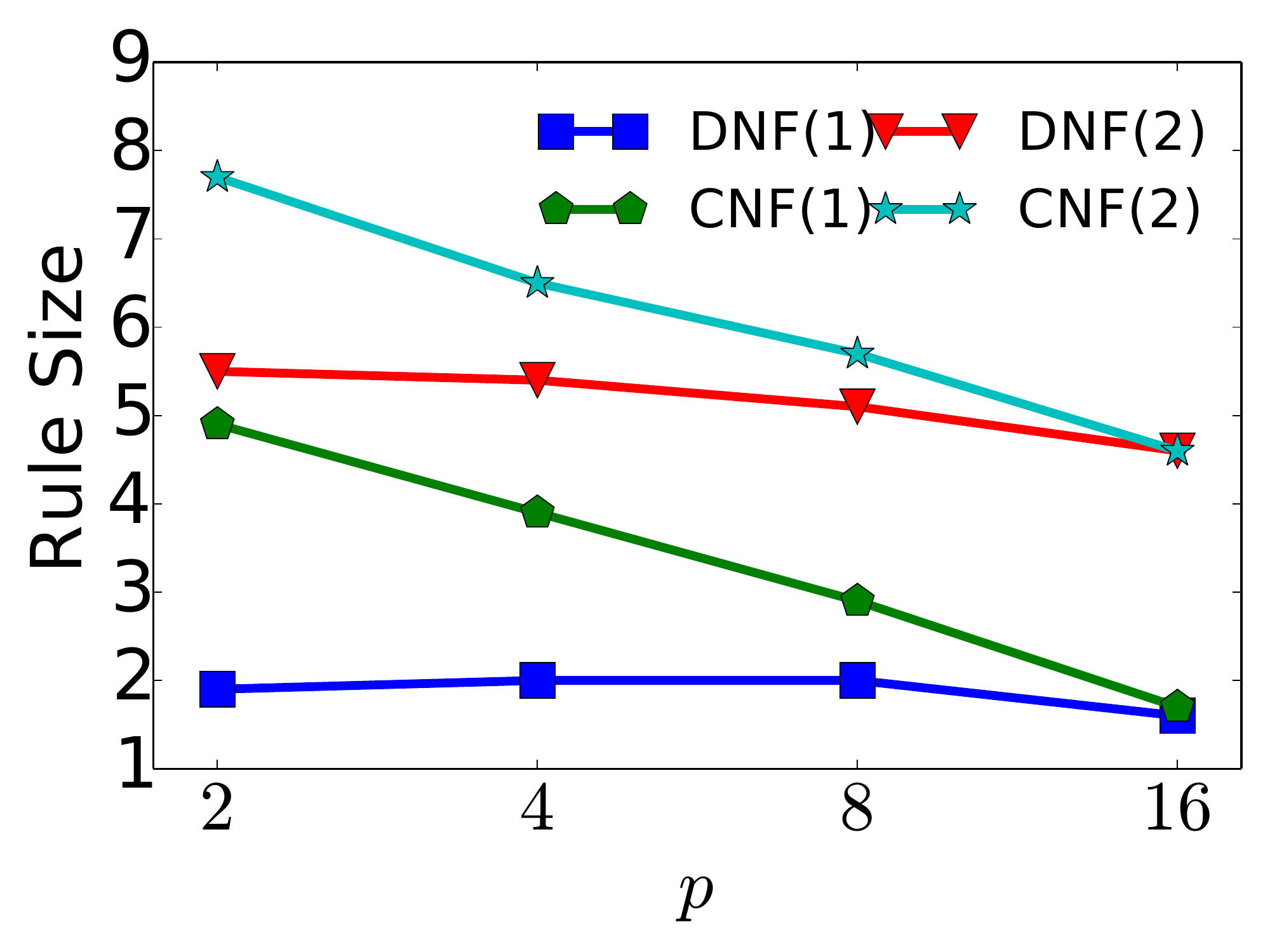}\label{fig:partition(a)}} \hfill
	\subfloat[]{\includegraphics[width=0.23\textwidth]{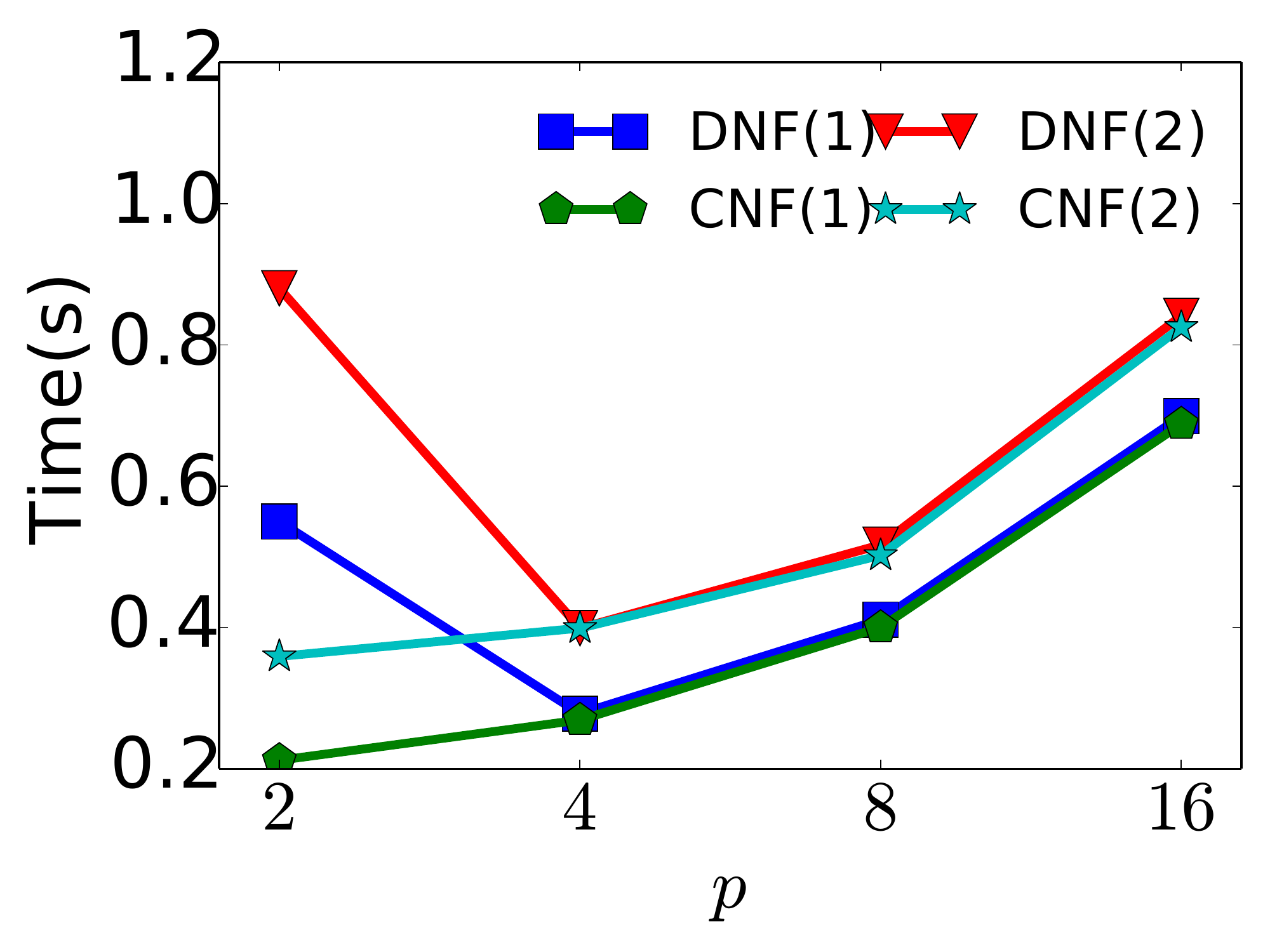}\label{fig:partition(b)}} \hfill
	\vspace{-12pt}
	\subfloat[]{\includegraphics[width=0.23\textwidth]{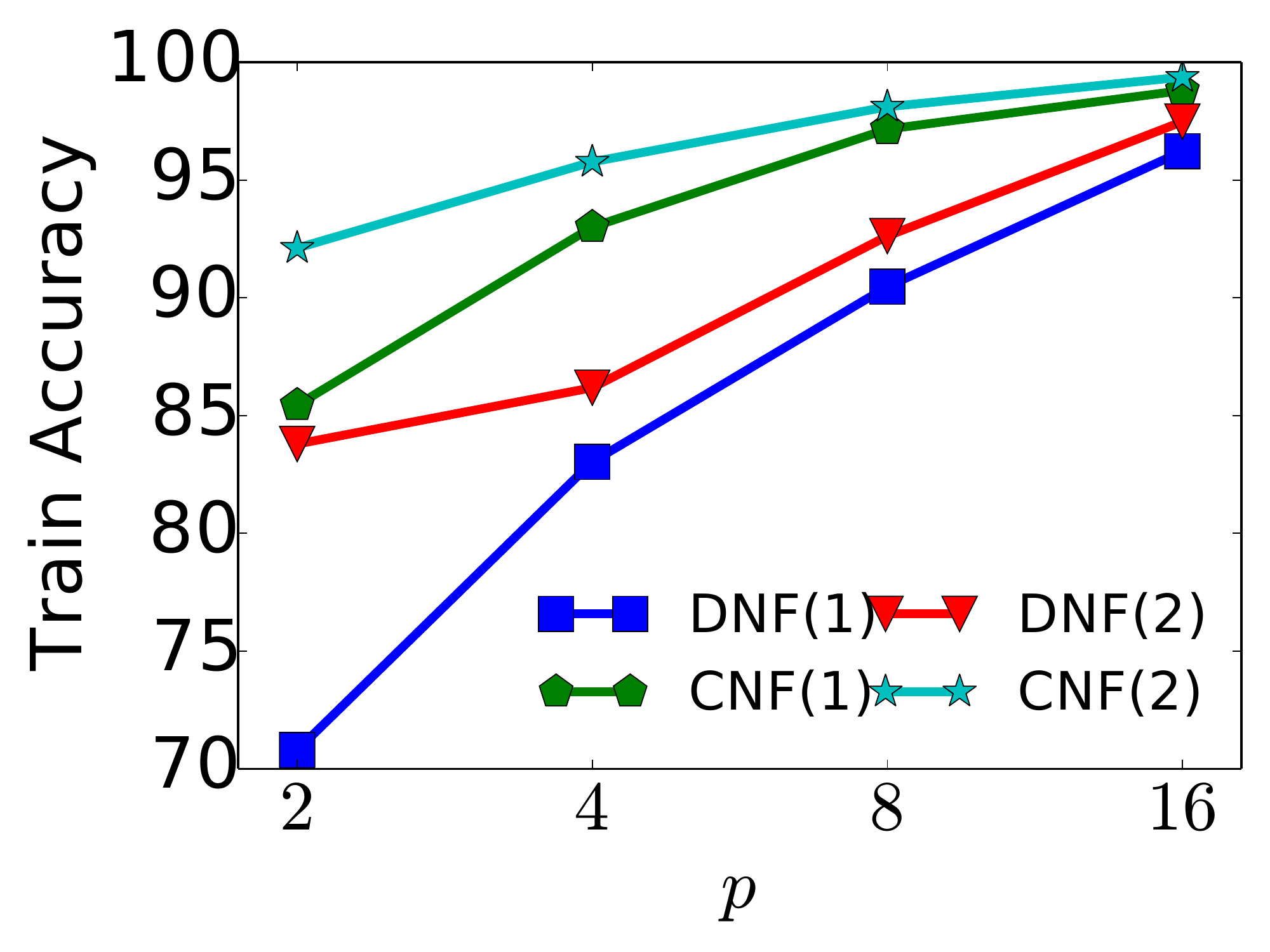}\label{fig:partition(c)}} \hfill
	\subfloat[]{\includegraphics[width=0.23\textwidth]{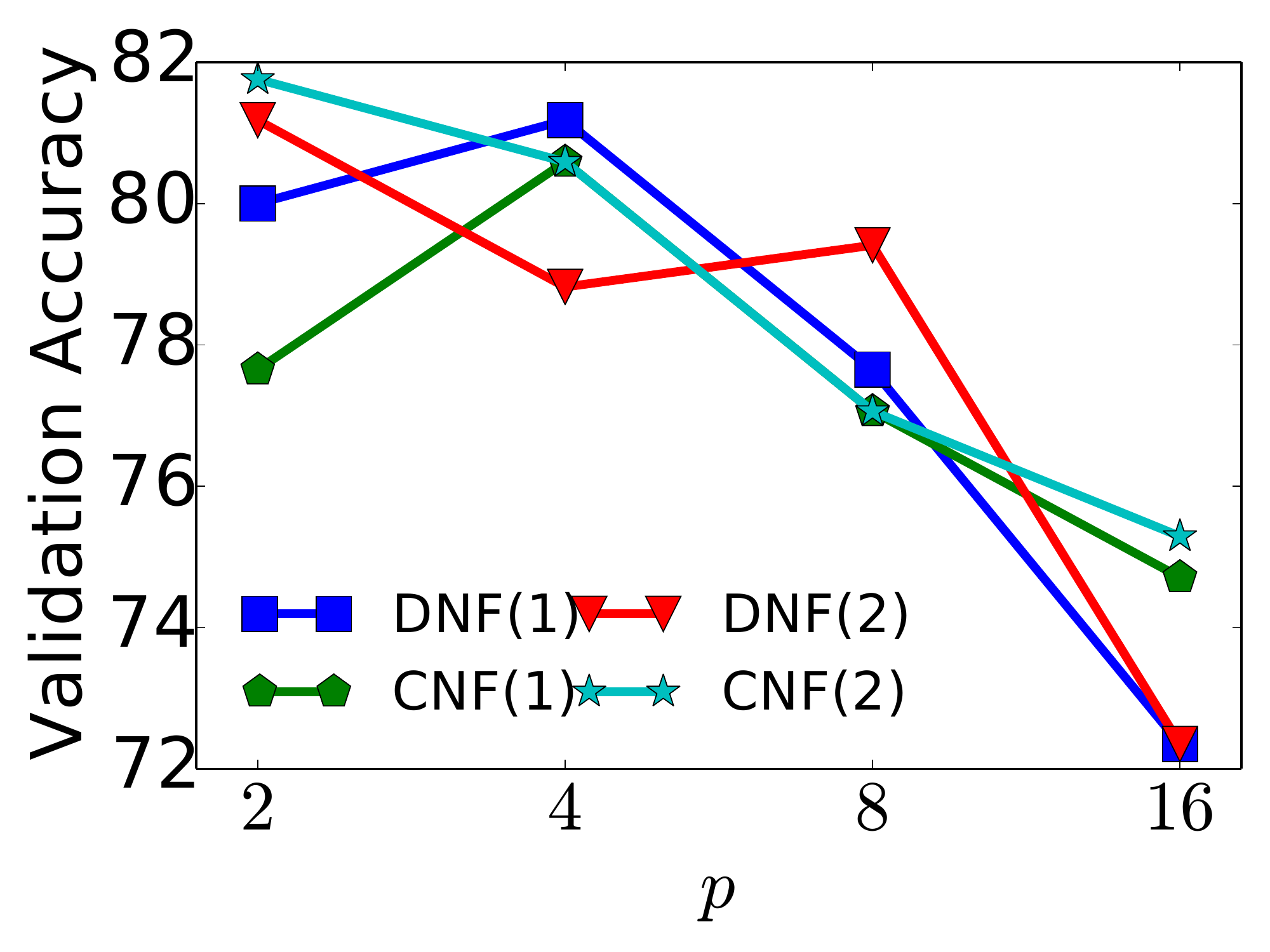}\label{fig:partition(d)}} \hfill
	\vspace{-8pt}
	\caption{Effect of the number of partitions $ p $ on rule size,  training time, train accuracy, and validation accuracy $ (k \in \{1,2\}, \lambda=10) $. The number within parenthesis denotes the number of clause $ k $ for the respective rule. 
	 }
 \vspace{-16pt}
	\label{fig:partition}
\end{figure}
\subsubsection{Varying The Number of Training Partitions $ p $:}
Figure ~\ref{fig:partition} presents the effect on rule size, training time, train accuracy, and validation accuracy as we vary $ p $. 
For all the datasets we find a similar observation  and here present the result for Parkinsons dataset for ease of exposition.  

In Figure \ref{fig:partition(a)} we observe that the size of the rule decreases as  $ p $ increases. This observation can be attributed to the decrease in the number of training samples per partition with the increase in the number of partitions and consequently, smaller rules suffice. 
In Figure \ref{fig:partition(b)}  $ \IMLI $ empirically shows that the  training time at first  decreases significantly and then increases slowly with the increase in $ p $. This observation can be attributed to the combined effect of the number of queries and the size of queries. Initially, we achieve a significant reduction in the size of query while the number of queries eventually dominate the runtime.  



 
  In Figure \ref{fig:partition(c)} we observe that $ \IMLI $ tends to make less training error as $ p $ goes higher because $ \IMLI $ learns on fewer  samples with fixed $ \lambda $ value. Moreover, we observe that CNF rules have higher train accuracy than  DNF rules, and $ 2 $-clause rules have higher train accuracy than $ 1 $-clause rules for both CNF and DNF rules.
  
  
   In Figure \ref{fig:partition(d)} we notice a  decrease in validation accuracy as we allow more partitions because learning on fewer samples  results in a rule that has less predictive power on validation set. For small $ p $, $ \IMLI $, however, ensures higher validation accuracy if the rule has more clauses.
   Moreover, effect of the number of partitions on test accuracy computed on unseen data does not follow any  specific pattern. 
   
    
    In summary, we observe that the number of partitions gives a sound handle to the end user to tradeoff the training time, validation accuracy, and interpretability of the rules.

%% file: conclusion.tex
\section{Conclusion}
\label{sec:conclusion}
In this paper, we present IMLI: an incremental framework for MaxSAT-based learning of interpretable classification rules. Extensive experiments on UCI datasets demonstrate that {\IMLI} achieves up to three orders of magnitude improvement in training time with only a minor loss of accuracy. We think {\IMLI} highlights the promise of MaxSAT-based approach and opens up several interesting directions of future research at the intersection of AI and SAT/SMT community. In particular, it would be an interesting direction of future research if the MaxSAT solvers can be designed to take advantage of incrementality of {\IMLI}.


%% file: Appendices.tex
\appendix
\input{related_work.tex}

	\section{Examples}
	\input{explanation.tex}

%
%
%
	
%
	\section{Proof}
	
	\begin{proof} (Lemma \ref{lm:adjustment_of_rule_size})
		\label{prf:prune}
		The case  $ |\R'|=|R| $ is trivial because no literal is removed from $ \R' $, and $ \R' $ is equivalent to  $ \R $.

		When $ |\R'| < |\R| $, $ \exists \langle b_i^l,b_j^l \rangle $  where $  l \in [1,k] $,    $siblings(b_i^l,b_j^l)=true $, $ tval(b_i^l) < tval(b_j^l) $, and $ x_c $ is the considered continuous feature. 
		
		Suppose  $op(b_i^l)=op(b_j^l)=\text{``}\ge \text{''} $. Consider two sets of real number: $ S_i= \{x_i: x_i \ge  tval(b_i^l)\} $ and $ S_j=\{x_j: x_j \ge  tval(b_j^l)\} $. As $ \R $ is in CNF,  $ \R $ checks $ I[x_c \in S_i \cup S_j] $  to classify  $ \forall q,\mathbf{X}_q  $. Here $ S_i \cup S_j = S_i $, so $ b_j^l $ can be pruned.
		
		We can prove similarly when $op(b_i^l)=op(b_j^l)=\text{``}< \text{''} $.
		
		Therefore, $ |\R'| \le |\R| $ and $ \R' $ is equivalent to $ \R $.
		
	\end{proof}
	
\section{Experiment}
\subsection{Dataset Description}
\label{sec:dataset}
We use nine publicly available datasets  of various size from  UCI repository  for conducting experiments for {\IMLI}. The datasets contain both real and categorical valued features. The datasets are  buzz events from two different social networks: Twitter and Tom's HW (Tom), Adult data (Adult), Parkinson's Disease detection dataset (Parkinsons), Ionosphere (Ion), Pima Indians Diabetes (PIMA), Blood service centers (Blood), breast cancer Wisconsin diagnostic (WDBC), and  Credit-default approval dataset (Credit-default).

	\section{Interpretable Rules}
	\label{sec:rules}
	In this section we are presenting the rules generated by $ \IMLI $ for the datasets we use in experiment. 
	
	\subsection{Rule for Credit Default Dataset:}
	A client will default if $ := $\\
	 ({education}$  =  ${others} OR  \\
	 {repayment status September: payment delay} $ > 1 $  {month} OR  \\{repayment status August: payment delay} $ > 2 $ {months} OR  \\{repayment status June: payment delay} $ > 2 $  { months})

	 \subsection{Rule for Adult Dataset}
	 
	 A person's income is greater than $ 50  $k if $ := $\\
	 (workclass is not   Federal-gov AND workclass is not   State-gov AND education is not   11th AND education is not   5th-6th AND education is not   7th-8th AND education-num $>$ 10.0 AND marital-status is not   Divorced AND marital-status is not   Married-AF-spouse AND marital-status is not   Married-spouse-absent AND marital-status is not   Never-married AND marital-status is not   Separated AND occupation is not   Handlers-cleaners AND occupation is not   Machine-op-inspct AND occupation is not   Priv-house-serv AND occupation is not   Protective-serv AND relationship is not   Own-child AND relationship is not   Unmarried AND native-country is not   Cambodia AND native-country is not   Columbia AND native-country is not   Dominican-Republic AND native-country is not   Guatemala AND native-country is not   Hungary AND native-country is not   Jamaica AND native-country is not   Laos AND native-country is not   Mexico AND native-country is not   Outlying-US,Guam-USVI-etc AND native-country is not   Poland AND native-country is not   Vietnam)

	\subsection{Rule for WDBC Dataset}
	
	Tumor is diagnosed as malignant if $ := $\\
	(standard area of tumor $ > 38.43  $ OR   \\largest perimeter of tumor  $ > 115.9 $  OR   \\largest number of concave points of tumor $ > 0.1508 $)
		
	\subsection{Rule for Blood Transfusion Service Center Dataset}

	He/she will donate blood if $ := $\\
	(Months since last donation $ \le 4 $ AND \\total number of donations $ > 3 $ AND \\total donated blood  $ \le 750.0 $ c.c. AND \\months since first donation $ \le 45 $)

	\subsection{Rule for Pima Indians Diabetes Database}
	
	Tested positive for diabetes if $ := $ \\
	(Plasma glucose concentration $ > 125 $ AND \\Triceps skin fold thickness $ \le 35 $ mm AND \\Diabetes pedigree function $ > 0.259 $ AND \\Age $ > 25 $ years) 
	
	\subsection{Rule for Parkinson's Disease Dataset}
	
	A person has Parkinson's disease if $ := $ \\	
	(minimum vocal fundamental frequency $ \le $ $ 87.57 $ Hz OR minimum vocal fundamental frequency $ > $ $ 121.38 $ Hz OR Shimmer:APQ3 $ \le $ $ 0.01  $ OR MDVP:APQ $ > $ $ 0.02 $ OR  D2 $ \le $ $ 1.93 $  OR NHR $ > $ $ 0.01  $ OR HNR $ > $ $ 26.5 $ OR spread2
	$ > $ $ 0.3 $) \\
	AND \\
	(Maximum vocal fundamental frequency $ \le $ $ 200.41 $ Hz OR HNR 
	$ \le $ $ 18.8 $ OR spread2 $ > $ $ 0.18 $ OR D2 $ > $ $ 2.92 $)
	
	\subsection{Rule for Ionosphere Dataset}

	  A radar is ``Good'' if$ := $ \\ 
	  (x$ 1 $ $ =1 \; $ AND $  $ x$ 2 $ $ > 0 \; $ AND $  $ x$ 4 $ $ > 0\;  $ AND $  $ x$ 5 $ $ > -0.23$ )
	
	 Here ``x'' represents the set of columns of the dataset.	
	
	\subsection{Rule for Tom's Hardware Dataset} 
	A topic is popular if $ := $ \\
	(Number of displays at time $ 2\; > 1936 \; $ OR $ $ Number of displays at time $ 7\; > 1250.6 $)

	\subsection{Rule for Twitter Dataset}
	
	A topic is popular if $ := $ \\
	(Number of Created Discussions at time $ 1 > \;78\; $ OR $ $ \\ Attention Level measured with number of authors  at time $6\; > 0.000365 \; $ OR $ $ Attention Level measured with number of contributions at time $ 0 \; > 0.00014 \; $ OR $ $ Attention Level measured with number of contributions at time $ 1 > 0.000136 \; $ OR $ $ Number of Authors at time $ 0 > 147 \; $ OR $ $ Average Discussions Length at time $3  > 205.4 \; $ OR $ $ Average Discussions Length at time $ 5 > 654.0$)

%% file: related_work.tex
\section{Related Work}
\label{sec:related}


The study of designing interpretable machine learning classifiers can find its root in the development of  popular models such as  decision trees \cite{bessiere2009minimising,quinlan2014c4}, decision lists \cite{rivest1987learning}, classification rules \cite{cohen1995fast} etc.  Apart from designing models with the purpose of generating interpretable rules, various studies have been conducted in order to  improve  the efficiency and scalability of the model. Specifically,  decision rule approaches such as  C4.5 rules \cite{quinlan2014}, CN2\cite{ClarkN1989}, RIPPER \cite{cohen1995fast}, SLIPPER  \cite{CohenS1999} rely on  heuristics, branch pruning, ad-hoc local criteria e.g., maximizing information gain, coverage, etc. because these models consider an intractable combinatorial optimization function.

In recent work,  
Malioutov et al.   has proposed rule based classification system by borrowing ideas from Boolean compressed sensing \cite{malioutov2013exact}. Two-level Boolean rules \cite{su2015interpretable} has been proposed to trade classification accuracy and interpretability, where hamming loss is used to characterize accuracy and sparsity is used  to characterize interpretability.
Wang et al. \cite{wang2015falling} has proposed a Bayesian framework for learning falling rule lists which is an ordered list of if-then rules. Chen et al. designs an optimization approach to learning falling rule lists and ``softly'' falling rule lists, along with Monte-Carlo search algorithms that use bounds on the optimal solution to prune the search space.

Incremental learning techniques are one possible solution to the scalability problem, where data is processed in parts, and the result combined so as to use less memory \cite{syed1999incremental}. 
 Incremental framework   has been studied in  SVM \cite{ruping2001incremental} to improve the existing approach. Specifically, an on-line recursive algorithm for SVM has been studied to facilitate  learning one vector at a time \cite{cauwenberghs2001incremental} and a local incremental approach has been proposed \cite{ralaivola2001incremental} to learn a SVM based on Radial Basis Function Kernel.

%% file: explanation.tex
\subsection{Illustration of Incremental Learning}
	\label{sec:example}
	We illustrate  an interpretable rule generated by $ \IMLI $ with step by step formulation over partitions on iris dataset\footnote{\url{https://archive.ics.uci.edu/ml/datasets/iris}}. Iris dataset has four attributes: sepal length, sepal width, petal length, and petal width. All feature values are scaled in centimeter. Iris dataset has three classes: Iris Setosa, Iris Versicolour, and Iris Virginica. We consider the binary problem of classifying Iris Versicolour  from the other two species: Setosa and Verginica. Here we consider that $ \R $ is a single clause  DNF rule and learned over four partitions (e.g. $ \R_1, \dots, \R_4 $). The final rule $ \R $  is equivalent to $ \R_4 $. 
	

\begin{equation*}
\begin{split}
&\R_1:=\text{petal length} \le 5.32 \,\wedge \, \text{petal length} > 1.7 \,\wedge \, \\
&\qquad\;\;\;\text{petal width} \le 1.8\\
&\R_2:=\text{sepal width} \le 3.1 \,\wedge \, \text{petal length} \le 5.32 \,\wedge \,\\ &\qquad\;\;\;\text{petal length} > 1.7 \,\wedge \, \text{petal width} \le 1.5\\
&\R_3:=\text{sepal width} \le 3.1 \,\wedge \, \text{petal length} \le 5.0 \,\wedge \,\\ &\qquad\;\;\;\text{petal length} > 1.7 \,\wedge \, \text{petal width} \le 1.5\\
& \R:=\text{sepal width} \le 3.1 \,\wedge \, \text{petal length} \le 5.0 \,\wedge \,\\ &\qquad\;\;\;\;\;\;\;\text{petal length} > 1.7 \,\wedge \, \text{petal width} \le 1.8
\end{split}
\end{equation*}
$ \R $ can be interpreted as: 
	a sample which satisfies all of the four constraints is predicted as Iris Versicolour. Here the  rule size $ | \R|=4 $. Specifically, $ \R_1 $ is learned on the first partition of training data. $ \R_2 $ learns two literals ($ 2^{nd} $ and $ 3^{rd} $) which also appear in $ \R_1 $, introduces two new literals ($ 1^{st} $ and $ 4^{th} $) while learning on the second training partition, and falsifies a previously learned literal  from $ \R_1 $ (i.e., $ (\text{petal width} \le 1.8 $). 
	
	Since the dataset contains continuous valued features, $ \IMLI $ removes redundant literals at each step of learning by applying  Algorithm \ref{alg:imli-RRL}. For example, we show the  generated rule $ \R_3' $  for the $ 3^{rd} $ partition if we do not apply Algorithm \ref{alg:imli-RRL}.

%
%

\begin{equation*}
\begin{split}
&\R_3':=\text{sepal width} \le 3.1 \,\wedge\, \text{petal length} \le 5.0 \,\wedge\, \\
&\qquad\;\;\;\underline{\text{petal length} \le 5.32} \,\wedge\, \text{petal length} > 1.7 \,\wedge\, \\
&\qquad\;\;\;\text{petal width} \le 1.5 \,\wedge\, \underline{\text{petal width} \le 1.8}\\
\end{split}
\end{equation*}

	The underline marked literals are redundant, hence removed in $\R_3 $.